\documentclass[conference]{IEEEtran}
\ifCLASSINFOpdf
\else
\fi

\usepackage{times}
\usepackage{url}
\usepackage[stable]{footmisc}
\usepackage{slashbox}
\usepackage{amssymb,amsthm,amsmath}
\usepackage{graphicx} 
\usepackage{epstopdf}
\usepackage{bm}
\usepackage{subfigure}
\usepackage{tablefootnote}

\newcommand\norm[1]{\left\lVert#1\right\rVert}
\newcommand\pd[2]{\frac{\partial #1}{\partial #2}}

\newcommand\he[3]{\frac{\partial^2 #1}{\partial #2\partial #3}}

\newtheorem{remark}{Remark}

\newtheorem{theorem}{Theorem}[section]
\newtheorem{lemma}[theorem]{Lemma}

\begin{document}\IEEEoverridecommandlockouts
\IEEEpubid{\makebox[\columnwidth]{This work has been presented at ICDM 2015 as a regular paper.} \hspace{\columnsep}\makebox[\columnwidth]{ }}
%
\title{A Unified Gradient Regularization Family for Adversarial Examples}

\author{\IEEEauthorblockN{Chunchuan Lyu}
\IEEEauthorblockA{Department of Computer Science \\and Software Engineering\\
Xi'an Jiaotong-Liverpool University\\
Email: chunchuan.lv@gmail.com}
\and
\IEEEauthorblockN{Kaizhu Huang}
\IEEEauthorblockA{Department of Electrical \\and Electronic Engineering\\
Xi'an Jiaotong-Liverpool University\\
Email:Kaizhu.Huang@xjtlu.edu.cn}
\and
\IEEEauthorblockN{Hai-Ning Liang}
\IEEEauthorblockA{Department of Computer Science \\and Software Engineering\\
Xi'an Jiaotong-Liverpool University\\
Email:haining.liang@xjtlu.edu.cn}}


%


\maketitle

\begin{abstract}
\textit{Adversarial examples} are augmented data points generated by imperceptible perturbation of input samples. They have recently drawn much attention with the machine learning and data mining community. Being difficult to distinguish from real examples, such adversarial examples could change the prediction of many of the best learning models including the state-of-the-art deep learning models. Recent attempts have been made to build robust models that take into account adversarial examples. However, these methods can either lead to performance drops or lack mathematical motivations. In this paper, we propose a unified framework to build robust machine learning models against adversarial examples. More specifically, using the unified framework, we develop a family of gradient regularization methods that effectively penalize the gradient of loss function w.r.t. inputs. Our proposed framework is appealing in that it offers a unified view to deal with  adversarial examples. It incorporates another recently-proposed perturbation based approach as a special case. In addition, we present some visual effects that reveals semantic meaning in those perturbations, and thus support our regularization method and provide another explanation for generalizability of adversarial examples. By applying this technique to Maxout networks, we conduct a series of experiments and achieve encouraging results on two benchmark datasets. In particular,we attain the best accuracy on MNIST data (without data augmentation)  and competitive performance on CIFAR-10 data.
\end{abstract}




%
\IEEEpeerreviewmaketitle

\section{Introduction}
Imperceptible perturbations in images are able to change the prediction of neural network models, including single layer softmax model~\cite{intr,goodfellow2014explaining,nguyen2014deep}. That is, given a trained neural network model and an input image, one can always find a small perturbation that can change the model's predication through certain optimization procedure. Moreover, perturbations trained in one model could also change the prediction results of many other classification models. The examples constructed by using such perturbations are referred to as \textit{adversarial examples}, which have attracted much interest in both machine learning and data mining~\cite{goodfellow2014explaining, advlearn,dm}.

The existence of adversarial examples raises several important issues. First, why do such perturbations exist in the first place? Intuitively, a successful machine learning model should robustly classify indistinguishable inputs as the same class. The generalizability  of those perturbations is even more intriguing, since those perturbations are obtained by optimization process based on model and image instance. Second, why do such perturbations not  occur frequently in real applications (i.e., naturally generated data)? Indeed, if most examples are adversarial examples, no machine learning algorithm could work at all. Finally and most importantly, what can we do to deal with adversarial examples? One possibility is to build machine learning models that are immune to adversarial examples. Alternatively, we might be able to use adversarial examples to even improve the performance of most machine learning models.

To explain the existence of adversarial examples,~\cite{goodfellow2014explaining} argued that this particular phenomenon could arise naturally from high dimensional linearity, as opposite to nonlinearity suspected by~\cite{intr}. ~\cite{analysis} showed that models' robustness against adversarial examples are limited by distinguishability between classes. This also supports that nonlinearity is not the fundamental reason behind adversarial examples. They also argued the generalizability of adversarial examples is due to neural network models' resemblance of linear classifiers. In addition to directly addressing the problem,~\cite{nguyen2014deep} have discovered a twin problem: there exist human unrecognizable images, which deep learning models could classify them with high confidence. There is not yet a verifiable answer to why adversarial examples exist only infrequently in real-life cases, but some speculate that they appear only in low probability regions in data manifolds \cite{intr,goodfellow2014explaining}. To alleviate the influence caused by adversarial examples, \cite{gu2014towards} tried to penalize the Jacobian matrix based on  a series of approximations. While their model seems more robust, it usually leads to an accuracy drop and slow training due to additional cost. Accuracy can be improved by injecting those perturbed examples back in training, as demonstrated by~\cite{intr}.~\cite{goodfellow2014explaining} have further extended this idea to the so-called \textit{fast gradient sign method}, and when using this method, trained models are more robust against adversarial examples. While this method seems promising, it appears to lack mathematical motivation and failing to fully utilize the idea of linear perturbation as a consequence. In data mining community, researchers also try to build models that are robust against adversarial examples (adversarial attacks)~\cite{advlearn,dm}. In particular, a study investigated correspondence between adversarial examples and effective regularizer, but their results are restricted merely for some specific loss functions~\cite{advlearn}.

We develop the linear view of adversarial examples proposed by~\cite{goodfellow2014explaining} in a more rigorous and unified way. To this end, we  propose a unified framework to train models that is robust to adversarial examples and successfully transform it to a minmax problem. Specifically, we propose a family of gradient based perturbations, as a unified regularization technique.  Models trained by applying our proposed gradient regularization family have proved highly robust to perturbations. Moreover, under different values of norm parameter, the family presents itself a unified framework, i.e.- it incorporates the recently-proposed~\textit{fast gradient sign method}~\cite{goodfellow2014explaining} as a special case and can also derive many other more promising methods. One interesting special case is verified to achieve encouraging performance in two benchmark data sets. Furthermore, by magnifying perturbations in MNIST~\cite{mnistlecun}, we could provide physical intuition for why adversarial examples could be generalized across different machine learning models.

\section{Gradient Regularization Family}\label{sec:grf}
In this section, we will present our framework of \textit{gradient regularization family} and describe its theoretical properties.

\subsection{General Framework}

We introduce our framework by starting with  the worst-case perturbation, initially proposed in~\cite{goodfellow2014explaining}. One salient feature of our framework is that it can represent a significant extension to the previous work into a unified family that incorporates many important variants of gradient regularization technique.

 Denote $\mathcal{L}(\bm x;\bm{\theta)}$~\footnote{This notation might be shorthanded as $\mathcal{L}(\bm x)$ or $\mathcal{L}$ in the following text.} as a loss function, $\bm x$ as data, and $\bm\theta$ as model parameters. The idea could be formalized as follows. Instead of solving $\min_{\substack{\bm\theta}}\mathcal{L}(\bm x;\bm{\theta)}$, ideally we would like  to solve the following problem  if we try to build a robust model against any small perturbation defined as ${\bm \epsilon}$:
 \begin{eqnarray}\label{eq:orginal}
 \min_{\substack{\bm\theta}}\max_{\substack{\bm{\bm \epsilon}:\norm{\bm{\bm \epsilon}}_p\le \sigma}}\mathcal{L}(\bm {x+{\bm \epsilon}};\bm{\theta)}
 \end{eqnarray}
The norm constraint in the inner problem implies that we only require our model to be robust against certain small perturbation. Thus the training procedure is decomposed into two stages. We first find a perturbation that maximizes the loss given the data and constraints. Then, we perform our ordinary training procedure to minimize the loss function by altering $\theta$. In general, this problem is difficult to be solved due to its non-convex nature with respect to $\epsilon$ and $\theta$.

In the following, we will propose to solve the problem using approximation technique. To this end, we first approximate the loss function by its first order Taylor expansion at point $x$. The inner problem then becomes:
\begin{eqnarray}\label{eq:simplified}
\max_{\substack{\bm{\bm \epsilon}}}\mathcal{L}(\bm x)+\nabla_{\bm x}\mathcal{L}^T\bm{\bm \epsilon} \qquad\mbox{s.t.}\quad\norm{\bm{\bm \epsilon}}_p\le \sigma
\end{eqnarray}
This problem is trivially linear, and hence convex w.r.t. ${\bm \epsilon}$. We can obtain a closed form solution by Lagrangian multiplier method, see~\ref{sec:solv} for details. This yields:
\begin{eqnarray}
 {\bm \epsilon} = \sigma \mbox{ sign} (\nabla\mathcal{L}) (\frac{|\nabla\mathcal{L}|}{\norm{\nabla\mathcal{L}}_{p^*}})^{\frac{1}{p-1}}
 \end{eqnarray}
where $p^*$ is the dual of $p$, i.e., $\frac{1}{p^*}+\frac{1}{p}=1$.

If we substitute the optimal ${\bm \epsilon}$ back to the original optimization problem, we can see that the influence of perturbations can be formulated as a regularization term. Thus, the new family of regularization method works approximately as:
\begin{eqnarray}\label{eq:regularizer}
  \min_{\substack{\bm\theta}}\mathcal{L}(\bm x)+\sigma\norm{\nabla_{\bm x} \mathcal{L}}_{p^*}
\end{eqnarray}
Instead of minimizing $\mathcal{L}(\bm x)$, we try to minimize $\mathcal{L}(\bm x+\bm \epsilon)$ where ${\bm \epsilon}$ is defined above, and parameterized by $p$ such that the new optimization objective could be highly robust to any small perturbations.

\begin{remark}
It is worthwhile to note that although it has been long known in~\cite{Bishop94trainingwith}, that injecting Gaussian noise is equivalent to penalizing the trace of Hessian matrix,  to our best knowledge, our approach is the first general method to penalize the gradient of loss function w.r.t. the input in complex models.  Such regularization could be applied into various machine learning model that contains gradient information.
\end{remark}

In the following, we examine three special cases of this family and show how this unified regularization framework could contain another method~\cite{goodfellow2014explaining}.

\subsection{Case $p = \infty$}

We show that in this case our method can be  reduced to the \textit{fast gradient sign method} proposed in~\cite{goodfellow2014explaining}. The worst perturbation ${\bm \epsilon}$ becomes:
 \begin{align}
 {\bm \epsilon} = & \sigma \lim\limits_{p\rightarrow\infty}\mbox{ sign} (\nabla\mathcal{L}) (\frac{|\nabla\mathcal{L}|}{\norm{\nabla\mathcal{L}}_{p^*}})^{\frac{1}{p-1}}
 \\= & \sigma\mbox{ sign} (\nabla\mathcal{L}) (\frac{|\nabla\mathcal{L}|}{\norm{\nabla\mathcal{L}}_{1}})^{0}
 \\&\mbox{(assuming $|\nabla\mathcal{L}|>0$)}
 \\= & \sigma\mbox{ sign} (\nabla\mathcal{L})
 \end{align}
The corner case, where $|\nabla\mathcal{L}|=0$, defines sign$(0)$ as 0. Therefore, we can  reduce our general method to a special case, namely \textit{fast gradient sign method}. The corresponding induced regularization term is $\sigma \norm{\nabla \mathcal{L}}_1$. Mathematically speaking, this regularization term appears to be unnatural, since the gradient is not penalized in an isotropic way. This might affect negatively the performance, especially when the data is preprocessed to be Gaussian-like.

\subsection{Case $p = 1$}

This is another special case, where $p=1$, and thus the worst perturbation ${\bm \epsilon}$ becomes:
 \begin{align}
 {\bm \epsilon} = & \sigma \lim\limits_{p\rightarrow1}\mbox{ sign} (\nabla\mathcal{L}) (\frac{|\nabla\mathcal{L}|}{\norm{\nabla\mathcal{L}}_{p^*}})^{\frac{1}{p-1}}
 \\= & \sigma\mbox{ sign} (\nabla\mathcal{L}) (\frac{|\nabla\mathcal{L}|}{\norm{\nabla\mathcal{L}}_{\infty}})^{\infty}
 \end{align}
 Therefore,
 \begin{align}
   {\bm \epsilon}_i = \begin{cases}
    \sigma, & \nabla \mathcal{L}_i  = \max_{\substack{j}}\mathcal{L}_j\\
    0, & otherwise \\
   \end{cases}
    \end{align}
 The intuition here is quite clear. Since we are constrained by the sum of absolute value of all perturbations, it is intuitive to put all of our ``budgets'' into one direction. However, the induced regularization term $\norm{\nabla \mathcal{L}}_\infty$ is not very appealing, since it only penalizes gradient in one direction.

\subsection{Case $p =2$}

As \cite{lp} have indicated that the extreme case of hyper parameter $p$ ($p=1$ or $p=\infty$) might not be the optimal setting, we would like to introduce \textit{standard gradient regularization} where $p=2$. Then, we will develop a second order Taylor expansion analysis of this perturbation to provide further theoretical insight into this particular case.
First, let us write down the formula for ${\bm \epsilon}$:
 \begin{align}
 \begin{split}
 {\bm \epsilon} = & \sigma \mbox{ sign} (\nabla\mathcal{L}) (\frac{|\nabla\mathcal{L}|}{\norm{\nabla\mathcal{L}}_{2}})^1
 \\=& \sigma \frac{\nabla\mathcal{L}}{\norm{\nabla\mathcal{L}}_2}
  \end{split}
 \end{align}
We are now ready to compute the second order expansion. The second order Taylor expansion of loss function is:

  \begin{center}
     \begin{align}
     \mathcal{L}(\bm x+\bm{{\bm \epsilon}})\approx& \mathcal{L} + \sigma\norm{\nabla_{\bm x} \mathcal{L}}_{2}+\frac{1}{2}\bm{\bm \epsilon}^T H_{\mathcal{L}}(x)\bm{{\bm \epsilon}}\\
     \approx&\mathcal{L} + \sigma\norm{\nabla_{\bm x} \mathcal{L}}_{2}+\frac{1}{2}\bm{\bm \epsilon}^T J_y(x) H_\mathcal{L}(y)J_y^T(x)\bm{{\bm \epsilon}}
     \\
           \intertext{(by Lemma~\ref{sec:lemma})}
              =&\mathcal{L} + \sigma\norm{\nabla_{\bm x} \mathcal{L}}_{2}+\frac{1}{2}\bm{\bm \epsilon}^T \nabla_x\mathcal{L}\nabla_x\mathcal{L}^T\bm{{\bm \epsilon}} \\=&\mathcal{L} + \sigma\norm{\nabla_{\bm x} \mathcal{L}}_{2}+\frac{\sigma^2}{2}\frac{\nabla\mathcal{L}^T \nabla\mathcal{L}\nabla\mathcal{L}^T\nabla\mathcal{L}}{\norm{\nabla \mathcal{L}}_2^2}
     \\=&\mathcal{L} + \sigma\norm{\nabla_{\bm x} \mathcal{L}}_{2}+\frac{\sigma^2}{2}{\nabla\mathcal{L}^T\nabla\mathcal{L}}
     \\=&\mathcal{L} + \sigma\norm{\nabla_{\bm x} \mathcal{L}}_{2}+\frac{\sigma^2}{2}Tr({\nabla\mathcal{L}\nabla\mathcal{L}^T})
        \\=&\mathcal{L} + \sigma\norm{\nabla_{\bm x} \mathcal{L}}_{2}+\frac{\sigma^2}{2}Tr( J_y(x) H_\mathcal{L}(y)J_y^T(x))  \\
        \intertext{(by Lemma~\ref{sec:lemma})}
     \approx&\mathcal{L} + \sigma\norm{\nabla_{\bm x} \mathcal{L}}_{2}+\frac{\sigma^2}{2}Tr({H_{\mathcal{L}}(x)})
       \end{align}
            \end{center}
               Therefore, the second order Taylor expansion resembled the term induced by marginalizing over Gaussian noise.
               \begin{remark}
               It is possible to write down a generalized form of second order regularization term as $\frac{\sigma^2}{2}\norm{\nabla_{\bm x}\mathcal{L}}_{p^*}^{2}$. In this form, it does not provide further insight, as it appears to be duplication of the first order term in general.
               \end{remark}
                    \subsection{Exact Solution Is Non-Trivial}

                    As mentioned earlier, the original minmax problem~(\ref{eq:orginal}) is generally difficult to solve. However, it would be still of interest to see whether the original minmax problem could be solved without using approximation. It turns out that it is not easy to achieve, even in the simplest case of a linear regression mode. We show this in the following.

                    The optimization problem of finding the worst perturbation in the linear regression model can be formulated as follows:
                    \begin{eqnarray}\max_{\substack{\bm{\bm \epsilon}}}\norm{\bm t - \Theta \bm{(x+\epsilon)} } \qquad\mbox{s.t.}\quad\norm{\bm{\bm \epsilon}}_p\le \sigma
                    \end{eqnarray}
                    where $\bm t$ is the desired output, and $\Theta$ is the weight matrix. Hence, we are maximizing a convex function, which is non-trivial. One might hope that by applying the first order expansion, it could lead to a simple solution. However, while this approximately solves the maximization problem, the minimization problem becomes non-convex due to the additional regularization term. Therefore, it appears to be hard to obtain exact solution in this very simple model.
                    \subsection{Computational Cost}
                    The computational cost of injecting such adversarial noise comes from computing $\epsilon$, which requires computing $\nabla_x \mathcal{L}$ and a minimum overhead to convert $\nabla \mathcal{L}$ into $\epsilon$. A naive BP approach to compute $\epsilon$ will roughly double the time cost of the training. We implemented the algorithm using built in functionality provided by~\cite{theano}.
                   \vspace{-0.2cm}
\section{Visualization and Interpretation of Gradient Perturbation}\label{sec:Q2}
Traditionally, adversarial examples and corresponding perturbations are regarded as unintelligible to naked eyes~\cite{intr,goodfellow2014explaining}. However, we will visualize our gradient perturbations in the case of $p=2$, and thus provide both physical intuition of gradient perturbation and mathematical structure behind. The physical intuition, in particular, could support the effectiveness of gradient perturbation, and explain why adversarial examples could generalize across models. The model used to generate sample images in this section is a sigmoid network model trained on the MNIST dataset~\cite{mnistlecun}.
\subsection{Visualizing Adversarial examples}
It has been demonstrated that adversarial examples are generated by imperceptible or unrecognizable perturbations~\cite{intr,goodfellow2014explaining}. However, we will show that in some cases we could visualize those perturbations. Let us present randomly picked samples from MNIST in Figure~\ref{original} and the perturbed images by gradient perturbation with $p=2$ and $\sigma = 1$ in Figure~\ref{corrupted}.

\begin{figure}[!t]
\vskip 0.2in
\begin{center}
\includegraphics[scale=.5]{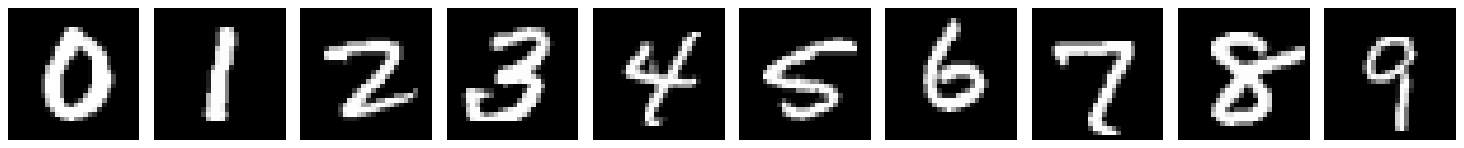}
\caption{Original input images from MNIST dataset}\label{original}
\end{center}
\begin{center}
\includegraphics[scale=.5]{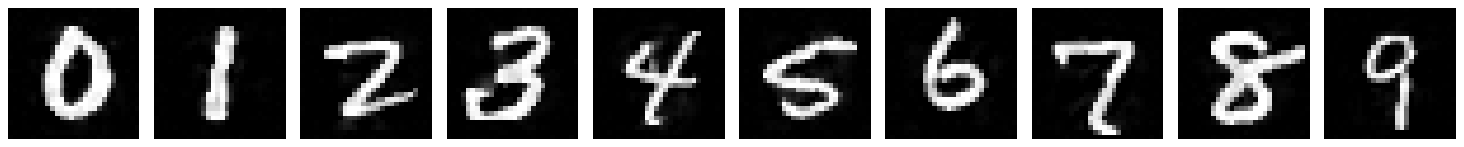}

\caption{Inputs perturbed by gradient perturbation with $p=2$ and $\sigma = 1$}\label{corrupted}
\end{center}
\vskip -0.2in
\end{figure}

Indeed, at the first glance, we may conclude that the perturbed examples are indistinguishable from the original copies. However, let us take a closer look at the perturbations in Figure~\ref{fig_pertu10_1} and Figure~\ref{fig_pertu10_2} by magnifying the perturbation with a factor of 10:
\begin{figure}[!t]
\vskip 0.2in
\begin{center}
\includegraphics[trim = 0 130 0 130,clip,scale=.5]{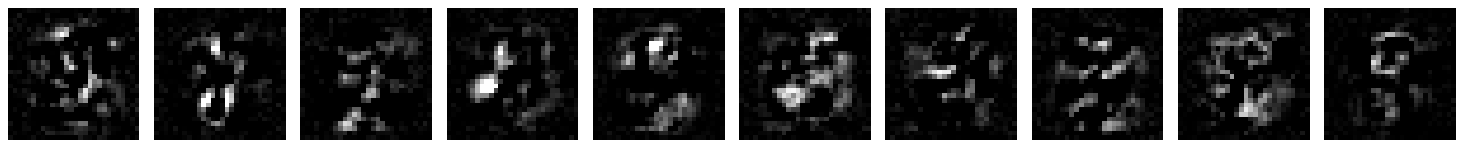}
\caption{The corresponding perturbations magnified by a factor of 10}\label{fig_pertu10_1}
\end{center}

\begin{center}
\includegraphics[scale=.5]{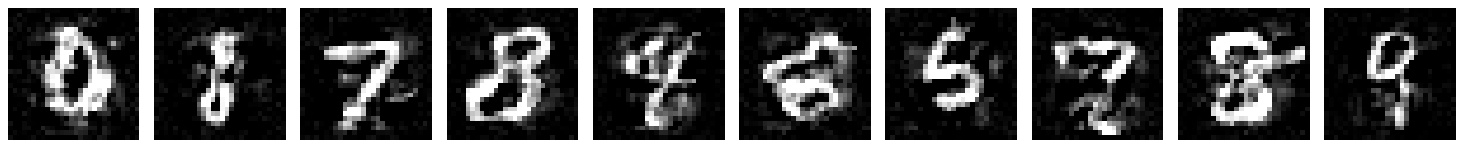}
\caption{Inputs perturbed by gradient perturbation with $p=2$ and $\sigma = 10$}\label{fig_pertu10_2}
\end{center}
\vskip -0.2in
\end{figure}
Although the magnified perturbations do not still make much sense to naked human eye, the exaggerated adversarial examples are presented in a more meaningful way for the naked eye (see Figure~\ref{fig_pertu10_2}). The perturbations have changed the physical shapes of the objects in a perceptible way. The number $2,3,6$ are of special interest. The bottom right corner of $2$ has been erased and turned into $7$; $3$ is extended into number $8$; the $6$ is the most interesting, where a part is erased while another part is extended to become a $5$.
\subsection{Mathematical Structure}
In this subsection, we will try to interpret the above phenomenon mathematically. To understand how those adversarial examples are generated, let us decompose the gradient perturbation $\bm{\bm \epsilon}$ when $p=2$:
\begin{align}
\bm{\bm \epsilon} = &\sigma  \frac{{\nabla_x\mathcal{L}}}{\norm{{\nabla_x\mathcal{L}}} _{2}}
\\= &\frac{\sigma }{\norm{{\nabla_x\mathcal{L}}} _{2}} {\nabla_x\mathcal{L}}
\\= &\frac{\sigma }{\norm{{\nabla_x\mathcal{L}}} _{2}} {\nabla_y\mathcal{L}}J_y(x)
\\= &\frac{\sigma }{\norm{{\nabla_x\mathcal{L}}} _{2}} \bm{(y-t)}J_o(x)
\end{align}
and where $\bm o$ is the input of the softmax layer for a typical neural network, and $\bm t$ is the true label. Next, we need to make the following assumption: the Jacobian matrix encodes the local version of corresponding classes. Then, when the prediction is concentrated on the correct output,  $(\bm{y-t})$ should be small and negative. This corresponds to the cases like $0,1,$ and $9$ in the previous section, where the perturbed examples are still clear and correct. In the case of low predication confidence, fuzzy perturbed examples could be generated. This corresponds to the cases like $4,5,7$, and $8$. The most interesting cases happen when there is a confusion between two classes. That is to say, $(\bm{y-t})$ contains two large components where one component is positive and one negative. The negative component corresponds to weak prediction of correct label, and the positive component corresponds to most confused wrong prediction. Images like $2,3$, and $6$ in the previous section are somewhat confusing when compared to the other numbers. In such cases, an adversarial perturbation erases the correct objects and injects the confused one.
\begin{remark}
In the above argument, we made the assumption of the Jacobian's ability to model the local space. This assumption is reasonable, since neurons in several layers of convolutional neural network correspond to certain natural image statistics~\cite{visualizing}. While such assumption might not hold strictly in more complex data or data in other domains, there are researches indicate that the encoding of such information exist~\cite{invert}.
\end{remark}

The above finding suggests that gradient perturbation method works in a distinctive way from techniques like weight decay or weight constraint method. It appears that gradient perturbation adds regularization effects conditioned on the performance of model in a particular instance. Therefore, it will not over-regularize when unnecessary. This interpretation is consistent with arguments made in~\cite{goodfellow2014explaining}.

\subsection{Generalization of Adversarial Examples}
The visualization tells us that image instances are actually morphed into other classes in a very minimal form. Our machine learning models are much more perceptive than people in detecting such minor perturbations. Thus, the ability of adversarial examples to generalize across different models might come from meaningful changes in examples. A small change towards a specific class will still be a change even added in another instance.

This explanation is also in coherent with~\cite{goodfellow2014explaining}. In their framework, we can say that the generalization of adversarial examples is due to the models' resemblance of linear classifier, and linear classifier could encode local semantic information of images.

\vspace{-0.2cm}
\section{Experiments}\label{sec:exp}
\vspace{-0.2cm}
In this section, we will apply our proposed gradient perturbation method on standard benchmarks including MNIST~\cite{mnistlecun} and CIFAR-10~\cite{cifar}. In addition, we will test whether gradient regularization could improve the models' robustness. We will mainly test the case for $p=2$, which we refer to as \textit{standard gradient regularization}.

\subsection{MNIST}
The MNIST dataset \cite{mnistlecun} consists of 60,000 training examples and 10,000 testing examples. Each single example consists of  $28\times28$ greyscale images, which corresponds to digits from 0 to 9. We rescale the inputs into the range of $[0,1]^{784}$, and no further preprocessing is applied. We tested \textit{standard perturbation regularization} on three architectures: standard sigmoid multilayer perceptron (MLP), Maxout network, and Convolutional Maxout network~\cite{maxout}. The sigmoid MLP experiments are conducted to investigate the utility of our method on highly non-linear models, and the other two experiments are designed to further verify whether our proposed framework could obtain the state-of-the-art performance.

The sigmoid MLP has two hidden layers, and its parameters are chosen from [400,400], [600,600] and [800,800] hidden units based on validation. Only norm constraint is used as other means for regularization. The max norm is set to be the square root of 15 as in~\cite{dropout}. The regularization parameter $\sigma$ is chosen from 0.1, 0.5, 1, and 2. The final architecture used is two layers of $600$ sigmoid units, and a softmax final layer. $\sigma$ is set to be 1. To fully utilize the training set, we followed~\cite{maxout} to train the model on 50,000 examples first, and we record the loss in the training set. Next, we train with a total of 60,000 examples until the last 10,0000 examples hit the same loss.

For the experiments based on Maxout networks, we have engaged the model from Maxout paper~\cite{maxout} and applied gradient regularization where $\sigma=1$ as we have found from experiments on sigmoid MLP. No other parameter setting is tested. In addition, we tested the case of $p = \infty$ for the Convolutional Maxout network, since it is not reported by~\cite{goodfellow2014explaining}. We have found that setting $\sigma = 0.1$ leads to better results than $\sigma = 0.25$, which is the best parameter for the non-convolutional Maxout network according to~\cite{goodfellow2014explaining}.

\begin{table}[!t]
\caption{Testing errors on permutation invariant MNIST dataset without data augmentation.}
\label{table-1}
\vskip 0.15in
\begin{center}
\begin{small}
\begin{tabular}{lcccr}
\hline
Models & Test Error  \\
\hline
Maxout + dropout & 0.94 $\%$\\
\cite{maxout}\\
Sigmoid mlp + gradient $p=2$ & 0.93 $\%$ \\
ReLu + dropout +gaussian & 0.85 $\%$ \\
\cite{noise}\\
Maxout + dropout + gradient $p=\infty$ & 0.84 $\%$\tablefootnote {It is recently improved to 0.78$\%$ by applying adversarial idea on early stopping, and changing architecture.}\\
\cite{goodfellow2014explaining}  \\
DBM + dropout &{0.79} $\%$\\
\cite{improve}\\
\textbf{Maxout + dropout + gradient p = 2 }&\textbf {0.78 $\%$}\\
\hline
\end{tabular}
\end{small}
      \end{center}
\vskip -0.1in
\end{table}

In the permutation invariant version, we achieved 78 errors among 10, 000 test samples in the case of $p=2$. \textbf{To our best knowledge, this is the best result in this category even with unsupervised pretraining if data are not augmented}. Also, gradient regularized old fashioned sigmoid MLP has been shown to achieve 93 errors, which actually beats simple Maxout networks with dropout training. This indicates that even a highly nonlinear model such as sigmoid MLP could be linear enough locally to benefit from gradient regularization. We summarized our results and other related best results in Table~\ref{table-1}.

\begin{table}[!t]
\caption{Testing errors on MNIST dataset without data augmentation. Convolutional architecture is used.}
\label{table-2}
\vskip 0.15in
\begin{center}
\begin{small}
\begin{tabular}{lcccr}
\hline
Models & Test Error  \\
\hline
Conv. Maxout + dropout & 0.45 $\%$\\
\cite{maxout}\\
Conv. Maxout + dropout& 0.41 $\%$\\
 + gradient $p=\infty$\\
\textbf{Conv. Maxout + dropout }& \textbf{0.39 $\%$} \\
\textbf{+ gradient p = 2 }\\
\textbf{Conv. Kernel Network} &\textbf{0.39 $\%$}\\
\textbf{\cite{kernel}}\\
\hline
\end{tabular}
\end{small}
\end{center}
\vskip -0.1in
\end{table}
In terms of the improvements on convolutional architecture, our best result is obtained by standard gradient regularization. We have achieved 39 errors on testing set, which ties with the recently proposed sophisticated convolutional kernel network described by~\cite{kernel}. To our best knowledge, no better results are obtained without using of data augmentation like elastic distortion. It appears that standard gradient regularization outperforms fast gradient sign method in a consistent fashion, especially if we consider that the $\sigma$ parameter is not tuned for Maxout network. We summarized the results in Table \ref{table-2}.

\subsection{CIFAR-10}
The CIFAR-10 dataset~\cite{cifar} consists of 60,000 32 $\times$ 32 RGB images, corresponding to 10 categories of objects. There are 50,000 training examples and 10,000 testing examples. Again, we have followed the procedure in the Maxout paper~\cite{maxout}. To quickly verify gradient regularization in a more complex setting, we tested on a small convolutional maxout network reported by~\cite{pylearn2}, and can confirm again the improvement of our regularization technique. $\sigma$ was set to be $1\times \sqrt{\frac{32^2\times3}{28^2}}\approx 2$ in this case, since the RBG image has more dimensions. Similar to MNIST, the training has two stages, we compared the test error in the end of both stages.
\begin{table}[!t]
\caption{Testing errors on CIFAR-10 dataset without data augmentation.}
\label{table-3}
\vskip 0.15in
\begin{center}
\begin{small}
\begin{tabular}{lcccr}
\hline
Models & Testing Error \\
\hline
Conv. Maxout Small + dropout & 14.05 $\%$\\
\cite{pylearn2} (40K)\\
{Conv. Maxout Small + dropout} &{13.26 $\%$}\\
{+Gradient p = 2 (40K)}\\
Conv. Maxout Small + dropout & 12.93 $\%$\\
\cite{pylearn2} (50K)\\
\textbf{Conv. Maxout Small + dropout }& \textbf{12.28 $\%$}\\
\textbf{+Gradient p = 2 (50K)}\\
\hline
\end{tabular}
\end{small}
\end{center}
\vskip -0.1in
\end{table}
As can be seen from Table \ref{table-3}, our gradient regularization technique improves the performance of this network.

\subsection{Robustness}
A recent theoretical study suggests there is a relationship between adversarial examples and Gaussian noise~\cite{analysis}, we have used testing errors under Gaussian noise as a criterion to test whether our regularization technique improves the models' robustness against perturbations. We will also discuss the relationship in ~\ref{sec:Q3}. We gather some data based on testing errors on MNIST.

\begin{table}[!t]
\caption{Testing errors of MNIST dataset under Gaussian noises}
\label{table-4}
\vskip 0.15in
\begin{center}
\begin{small}
\begin{tabular}{lcccr}
\hline
\backslashbox{Models}{Gaussian std} & 0 & 0.1 & 0.3 \\
\hline
Conv. Maxout + dropout & 0.45 $\%$ & 0.76 $\%$ & 20.17 $\%$\\
Maxout + dropout  & 0.94 $\%$ & 1.14 $\%$ & 2.42 $\%$\\
\hline\\
Conv. Maxout + dropout& 0.39 $\%$ & 0.44 $\%$ & 14.36 $\%$\\
 +Gradient p = 2\\
Maxout + dropout & 0.78 $\%$ & 0.83 $\%$ & 1.29 $\%$\\
 +Gradient p = 2\\
\hline
\end{tabular}
\end{small}
\end{center}
\end{table}
\begin{table}
\vskip -0.1in
\caption{Testing errors of CIFAR dataset under Gaussian noises}
\label{table-5}
\vskip 0.15in
\begin{center}
\begin{small}
\begin{tabular}{lcccr}
\hline
\backslashbox{Models}{Gaussian std} & 0 & 0.1 & 0.3 \\
\hline
Conv. Maxout + dropout & 13.08 $\%$ & 13.37 $\%$ & 16.88 $\%$\\
\hline\\
Conv. Maxout + dropout& 12.28 $\%$ & 12.81 $\%$ & 16.80 $\%$\\
 +Gradient p = 2\\
\hline
\end{tabular}
\end{small}
\end{center}
\vskip -0.1in
\end{table}
As can be seen from Table~\ref{table-4} and~\ref{table-5}, our regularization technique indeed improves the models' ability to resist perturbation. However, it is interesting to note that the convolutional architecture seems to suffer more severe than the standard Maxout network. This observation is consistent with the data gathered from the ImageNet data set~\cite{intr}, where researchers have shown that the operator norm is larger in the convolutional architecture than fully connected networks. This is however yet another indication that the randomness in natural images is quite small even compare to those so called imperceptible perturbations.

\section{Discussion}\label{sec:Q3}

In this section, we will show how to exploit the proposed unified theory and mainly discuss why adversarial examples do not occur frequently in real data. Some of the following analysis may not be strict enough but sufficient to provide intuition and interpretations on certain properties of  adversarial examples.

By assuming the perturbation is caused by the Gaussian noise, we will first analyze how the misclassification rates of various learning models could be related to the minimum perturbation. Here, the minimum perturbation is measured in terms of its perturbation size that is able to change the models' prediction on a certain example.  After the establishment of the relationship, we will be able to predict the probability or the frequency that adversarial examples occur in real data. In order to achieve these objectives, we will develop a mathematical model by following  two fundamental assumptions made in~\cite{goodfellow2014explaining}: the manifold where adversarial examples live is not special and they mainly arise from local linear behavior.

More precisely, we assume that every data point is associated with a Gaussian distribution:
$(\bm x+\bm\eta)\sim \mathcal{N} (\bm x,\sigma^2\bm I_d)$
where $\sigma$ is the standard deviation in every dimension, and $\bm I_d$ is an identity matrix of size $d\times d$. Thus, adversarial examples do not live in a region with particularly low probability as contrary to suggestion from~\cite{intr}. There are $n$ directions to generate effective adversarial examples, whose corresponding minimum perturbations are respectively denoted as $\bm a_i$, and direction $\bm e_i=\frac{\bm a_i}{\norm{\bm a_i}}$ for $i=1,2,\ldots,n$. In order to generate an adversarial example from $\bm\eta$, we need to satisfy:
\begin{eqnarray}\exists i, \quad\mbox{s.t.}\quad \bm\eta^T\bm e_i\geq  \bm a_i^T \bm e_i
\end{eqnarray}
Denote adversarial examples to occur as $A$, then we could obtain a lower bound:\begin{eqnarray}P(A)\leq \sum_{i=1}^n P(\bm\eta^T\bm e_i\geq  \bm a_i^T \bm e_i)
\end{eqnarray}
 Since the Gaussian distribution is isotropic, $P(\bm\eta^T\bm e_i\geq  \bm a_i^T \bm e_i)$ could be simplified into:
  \begin{eqnarray}\label{eq:p}
  P(\bm\eta_0\geq  \norm{\bm a_i})
  \end{eqnarray}
  and where $\bm\eta_0$ means $\bm \eta $ in an arbitrary axis. This essentially means that given $\norm{\bm a_i}$, it will not increase the probability of adversarial examples to occur by increasing dimension of input space. To calculate the probability, we only need $\norm{\bm a_i}$. To further simplify our formula, we assume all $\norm{\bm a_i}$'s are equal, and just denote each $\bm a_i$ as $\bm a$. Hence, there are three parameters in our model: $n,\norm{\bm a},\sigma$, among which the size of $\sigma$ is controlled.

  We could develop some qualitative intuitions first. To make a fair comparison, we set that the mean distortion of adversarial examples, as defined by \cite{intr}, equal to the variance of the Gaussian distribution. In other words, we set $\sigma^2=\frac{\norm{\bm a}_2^2}{d}$, so that our model only depends on the number of directions $n$ and the dimension of input $d$. Then, based on Equation~(\ref{eq:p}), we have:
  	\begin{eqnarray}P(A)\leq n P(\frac{\bm\eta_0}{\sigma}\geq \sqrt d)
  	\end{eqnarray}
  	 Since the event denotes $\bm\eta_0$ to be $\sqrt d$ standard deviation away from the mean, the probability of the event to occur shrinks exponentially. Therefore, while high dimensionality makes it possible to find a direction such that the model prediction would be changed, it also provides protection in a probabilistic sense that such prediction change could be unlikely.

  In the next subsection, we will concretely take the MNIST data as one real example to make further analysis. The image pixels used in the following analysis are all scaled into [0,1]. Following~\cite{intr}, we will conduct our analysis based on the simple softmax model, which is of linear nature.
  
\subsection{Estimating Misclassification Rate under Gaussian Noise}
This subsection mainly aims at verifying our linear analyzing framework and gathering insights from empirical observations. To get a reasonable numerical estimation of misclassification rate, it is not enough to use the average of minimum perturbation $\norm{a}_2$, instead we have to take the distribution of minimum perturbation into consideration. There are two reasons for this: (1) the minimum perturbations of a small value have a major impact on the probability that adversarial examples occur, and (2) the minimum perturbation skews towards small values as seen in Figure~\ref{fig:mini}.
\begin{figure}[!t]
\vskip 0.2in
\begin{center}
\includegraphics[scale=.3]{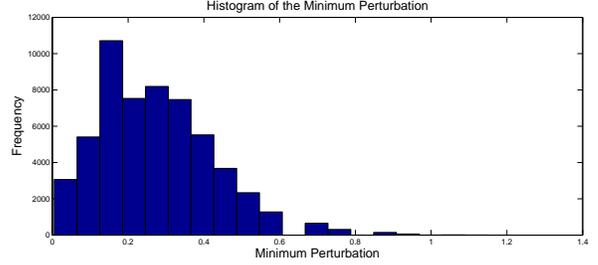}
\caption{Histogram of minimum perturbation in a single layer softmax model with regularization parameter $\lambda=1e-4$. Histograms of other regularization parameters are of the same shape but different in scale}\label{fig:mini}
\end{center}
\end{figure}
We approximate the distribution  of $\norm{\bm a}$ by a Gaussian distribution $\mathcal{N}_a (\mu_a,\sigma_a^2 )$. We will then have:
\begin{align}
P(A)\leq & n P(\bm\eta_0\geq{\norm{\bm a}})
\\ = & n P(\bm\eta_0-{\norm{\bm a}}\geq 0)
\end{align}
The summation of two Gaussian random variables, denoted by $\delta \equiv \bm\eta_0-{\norm{\bm a}}$, is  Gaussian, whose mean and variance are $-\mu_a$ and $\sigma^2+\sigma_a^2$ respectively. Then, we have the probability related to a cumulative distribution function of a Gaussian variable:
\begin{eqnarray}P(A) \leq   n P(\delta \geq 0)
\end{eqnarray}
If we further assume the smallest perturbation in one certain direction dominates the predicated probability, we can roughly set $n=1$.\footnote{Such assumption may be less strict, which is however simple and useful to get clear insight into the theory.} Then, the remaining task is to get an estimation of $\mu_a$ and $\sigma_a^2 $. Following the linear view of adversarial examples, we perform a line-search in the direction of gradient $\nabla_{\bm x}\mathcal{L}$ to find the minimum perturbation needed to change the prediction from correct to wrong. We have gathered the statistics from softmax models, and summarized the results in Table~\ref{table:mni}. Note that these statistics are obtained under the training set (with sufficiently large number of samples) so as to get a statistically reliable conclusion. Similar inspection can be derived if the test data were used.

\begin{table}[!t]
\caption{Mean and standard deviation of the minimum perturbations for softmax models on MNIST training dataset}
\label{table:mni}
\vskip 0.15in
\begin{center}
\begin{small}
\begin{tabular}{lccr}
\hline
Models & mean & standard deviation \\
\hline
Softmax $\lambda =1e-4$& 0.2744  & 0.1511 \\
Softmax $\lambda =1e-2$& 0.4722  & 0.2531 \\
Softmax $\lambda =1$& 1.2718 & 0.6654 \\
\hline
\end{tabular}
\end{small}
\end{center}
\vskip -0.1in
\end{table}

As can be seen from Table~\ref{table:mni}, a larger regularization parameter corresponds to better robustness. To estimate the misclassification rate under different levels of Gaussian noises, we first need to obtain the misclassification rate without noise. With the further assumption that misclassification is not corrected by noise, the predicated misclassification rate  can be calculated theoretically as:
 \begin{eqnarray}
 P(miss)+(1-P(miss))\times P(\delta \geq 0)\label{missrate}
\end{eqnarray}
Given the perturbation under the Gaussian noise, we can get the actual misclassification rates in MNIST and also estimate such rate theoretically using Equation~(\ref{missrate}).  These results are listed in Table~\ref{table:noise} and Table~\ref{table:pre} respectively.

\begin{table}[!t]
\caption{Acutal misclassification rate on MNIST under the perturbation of Gaussian noise}
\label{table:noise}
\vskip 0.15in
\begin{center}
\begin{small}
\begin{tabular}{lcccr}
\hline
\backslashbox{Models}{Gaussian std} & 0 & 0.1 & 0.3 \\
\hline
Softmax $\lambda =1e-4$& 6.02  $\%$ & 11.99 $\%$ &  40.07  $\%$\\
Softmax $\lambda =1e-4$& 6.29 $\%$ & 8.22 $\%$ & 22.44 $\%$\\
Softmax $\lambda =1e-4$& 8.93 $\%$ & 9.16 $\%$ & 11.33 $\%$\\
\hline
\end{tabular}
\end{small}
\end{center}
\vskip -0.1in
\end{table}
\begin{table}[!t]
\caption{Misclassification rate theoretically estimated by (\ref{missrate}) on MNIST  under the perturbation of  Gaussian noise}
\label{table:pre}
\vskip 0.15in
\begin{center}
\begin{small}
\begin{tabular}{lcccr}
\hline
\backslashbox{Models}{Gaussian std} & 0 & 0.1 & 0.3 \\
\hline
Softmax $\lambda =1e-4$& 6.02  $\%$ &12.12 $\%$ &  25.47  $\%$\\
Softmax $\lambda =1e-2$& 6.29 $\%$ & 10.16 $\%$ & 17.01 $\%$\\
Softmax $\lambda =1$& 8.93 $\%$ & 11.61 $\%$ & 12.63 $\%$\\
\hline
\end{tabular}
\end{small}
\end{center}
\vskip -0.1in
\end{table}
As can be seen by comparing Table~\ref{table:noise} and Table~\ref{table:pre}, the  misclassification rate estimated from our probability model fits well the actual rate as obtained from empirical experiments. Although these numerical results do not match exactly as seen in Table~\ref{table:noise} and Table~\ref{table:pre}, the qualitative changes of the predications and actual results are fairly consistent under different levels of Gaussian noises.

Note that, within the linear view of adversarial examples, we have made two major assumptions in the previous analysis. First, we adopt the most dominant direction to estimate the misclassification rate. This may lead to underestimation of misclassification, since other directions are neglected. Second, we approximate the distribution of minimum perturbation by a Gaussian distribution. This may lead to overestimation of misclassification rate in the case of no noise. Table~\ref{table:mni} indicates that standard deviation of minimum is roughly half of the mean in a consistent fashion. Then, even without Gaussian noise, the calculation yields roughly $P(A)\approx3 \%$. The reason behind this is that the actual distribution is only supported in the positive region, but Gaussian distribution is not. This simplification does not have a major impact when $\sigma$ is large, and our results agree with this. However, to estimate $P(A)$ when $\sigma$ is small, we may need to readjust our approximation.

\subsection{Estimating the Probability of Adversarial Examples}
To estimate the probability that adversarial examples occur in reality, we need to estimate the randomness in natural images and the distribution of the minimum perturbations. This is especially necessary when the perturbation is sufficiently small, i.e., close to $0$, since the overestimation occurs in this case as indicated in the previous subsection. In the following, we intend to choose a $\sigma$ that is sufficiently large, yet is hardly perceptible. We first illustrate some examples from MNIST dataset in Figure~\ref{fig_001}, where $\sigma=0.01$.

\begin{figure}[!t]
\vskip 0.2in
\begin{center}
\includegraphics[scale=.5]{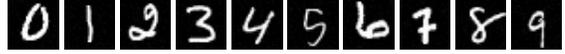}
\caption{Images corrupted by Gaussian noise with $\sigma=0.01$}\label{fig_001}
\end{center}

\vskip -0.2in
\end{figure}
In Figure~\ref{fig_001} , such noise is imperceptible even with a close examination, and this is not the case for significantly larger $\sigma$. We observe that the distribution of the minimum perturbation appears to be linear when it is close to zero as seen in Figure~\ref{fig:around0}.\footnote{The space in the figure is caused by steps in line search.}
\begin{figure}[!t]
\vskip 0.2in
\begin{center}
\includegraphics[scale=.3]{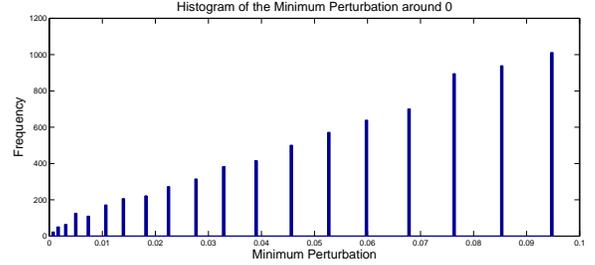}
\caption{Histogram of minimum perturbation in a single layer softmax model with regularization parameter $\lambda=1e-4$. Region beyond 0.1 is filtered.}\label{fig:around0}
\end{center}
\vskip -0.2in
\end{figure}
 In light of this, we could assume the probability density is linear around $0$ with respect to the minimum perturbation. After gathering statistics from the graph and some algebras, we can predicate the increase of misclassification rate in theory (based on our model) for a perturbation caused by Gaussian noise with $\sigma=0.01$. The percentage of such additional  misclassification  is calculated theoretically and listed in Table~\ref{table:ad}. We also list the actual additional misclassification, which was obtained empirically from experiments.
\begin{table}[!t]
\caption{Actual and predicated percentage of additional misclassification rate caused by the perturbation of Gaussian noise with $\sigma=0.01$ on MNIST}
\label{table:ad}
\vskip 0.15in
\begin{center}
\begin{small}
\begin{tabular}{lccr}
\hline
Models & Actual & Predicated \\
\hline
Softmax $\lambda =1e-4$& 0.07  $\%$ & 0.04 $\%$ \\
Softmax $\lambda =1e-2$& 0.01 $\%$ &  0.02 $\%$ \\
Softmax $\lambda =1$& 0.01 $\%$ & 0.01 $\%$ \\
\hline
\end{tabular}
\end{small}
\end{center}
\vskip -0.1in
\end{table}
Note that, according to the above analysis and assumptions, these additional misclassified samples can be regarded as adversarial examples that occur naturally, when a perturbation is given by Gaussian noise. As seen in Table~\ref{table:ad}, these numbers  are of sufficiently small both theoretically (under column ``Predicated) and empirically (under column ``Actual"), meaning that the probability of adversarial examples is quite small. Therefore, the claim that adversarial examples could endanger our machine learning models seems to somewhat over cautious, since its natural occurrence is supposed to be rare and negligible comparing to classification errors of our models. This is both verified theoretically and empirically in the above analysis.

\section{Conclusion}
In this paper, we have proposed a unified framework, which aims at building robust models against adversarial examples. We have derived a family of gradient regularization techniques, which could be extensively used in various machine learning models.  In particular, a special case has a second order interpretation, which approximately marginalized over Gaussian noise. We explained the adversarial examples' ability to generalize across models by visualizing its effects, which had been previously thought to be unrecognizable. This visualization, in turn, supplement our regularization technique with physical meaning.

While the linear view appears to be plausible and fruitful, so far, it needs more empirical support. To be more concrete, manifold hypothesis suggests data residue in a low dimensional manifold in the space. This, however, suggests that it is always possible to leave the manifold by slight change, since there are extra dimension around every point in the manifold, and the fact supports the hypothesis. Theoretically, a low dimensional manifold has measure of zero. This makes the number analogy made by~\cite{goodfellow2014explaining} almost literally true, but with adversarial examples being irrational numbers. In reality, however, data float around the manifold. Thus, geometrically, we could check the implication of linear view by associating errors induced by Gaussian noises and the minimum perturbation length. Our primary work suggests that it would require distribution of minimum perturbation length to account for the actual errors.


\section*{Acknowledgment}

The authors would like to thank the developer of Deeplearning Toolbox~\cite{deep}, the developers of Theano~\cite{theano,Theano-2012} and the developers of Pylearn2~\cite{pylearn2}.



%
\appendices
\section{Solving $\bm \epsilon$}\label{sec:solv}
Since $\mathcal{L}(\bm x)$ is independent of $\bm\epsilon$. Our problem is \begin{equation}
\max_{\substack{\bm{\bm \epsilon}}}\nabla_{\bm x}\mathcal{L}^T\bm{\bm \epsilon} \qquad\mbox{s.t.}\quad\norm{\bm{\bm \epsilon}}_p\le \sigma
\end{equation}Clearly, the optimal $\bm\epsilon$ would have a norm of $\sigma$, otherwise, we can normalize $\bm\epsilon$ to get a greater loss. Therefore, we are set to solve \begin{equation}\max_{\substack{\bm{\bm \epsilon}}}\nabla_{\bm x}\mathcal{L}^T\bm{\bm \epsilon} \qquad\mbox{s.t.}\quad\norm{\bm{\bm \epsilon}}_p= \sigma\end{equation}
This could be solved by standard Lagrangian multiplier method, where $f(\bm\epsilon)\equiv\nabla_{\bm x}\mathcal{L}^T\bm{\bm \epsilon}$, and $g(\bm\epsilon)\equiv\norm{\bm{\bm \epsilon}}_p= \sigma$. Set $\nabla f(\bm\epsilon)= \lambda\nabla g(\bm\epsilon)$, we have
 \begin{align}
  \nabla f(\bm\epsilon)&= \lambda\nabla g(\bm\epsilon)  \\
    \nabla_{\bm x}\mathcal{L}&=\lambda\frac{\bm\epsilon^{p-1}}{p(\sum_i \bm\epsilon_i^p)^{1-\frac{1}{p}}}\\ 
    \nabla_{\bm x}\mathcal{L}&=\frac{\lambda}{p}(\frac{\bm\epsilon}{\sigma})^{p-1}\label{epsilon}\\
      \nabla_{\bm x}\mathcal{L}^{\frac{p}{p-1}}&=(\frac{\lambda}{p})^{\frac{p}{p-1}}(\frac{\bm\epsilon}{\sigma})^{p}\\
          \end{align}
      Sum over two sides
       \begin{align}
         \sum   \nabla_{\bm x}\mathcal{L}^{\frac{p}{p-1}}&=\sum(\frac{\lambda}{p})^{\frac{p}{p-1}}(\frac{\bm\epsilon}{\sigma})^{p}\\
                      \norm{\nabla\mathcal{L}}_{p^*}^ {p^*} &=(\frac{\lambda}{p})^{p^*}*1\\
               (\frac{\lambda}{p})&=\norm{\nabla\mathcal{L}}_{p^*}\label{lp}
    \end{align}
    Combine~\eqref{epsilon} and~\eqref{lp}, it is easy to see
          \begin{equation}
     {\bm \epsilon} = \sigma \mbox{ sign} (\nabla\mathcal{L}) (\frac{|\nabla\mathcal{L}|}{\norm{\nabla\mathcal{L}}_{p^*}})^{\frac{1}{p-1}}
          \end{equation}
\section{Lemma}
Let vector $y$ be the output of model, i.e. $y= y(x)$. Correspondingly, we have vector $t$ as the desired output. We have $N$ as the length of $ y$, where $N$ is the number of classes in a classification task. $ GN(x)\equiv J_{ y}(x) H_\mathcal{L}(y)J_y^T(x)$ is Gaussian-Newton matrix, which is an approximation of Hessian matrix of $\mathcal{L}$ w.r.t. $x$~\cite{GN}. The $J_y(x)$ is the Jacobian matrix of $ y$ w.r.t. $  x$, and $H_\mathcal{L}(y)$ is the Hessian matrix of $\mathcal{L}$ w.r.t. $ y$.
 \begin{lemma}
 \label{sec:lemma}
 Given $ \mathcal{L}=-\sum_{i=1}^N t_i\log (y_i)$, and  $ \exists l\in 1,2,,\ldots ,N $ s.t. $ \forall i \in 1,2,,\ldots ,N, i=\delta_{li}$. We have $\nabla_x\mathcal{L}\nabla_x\mathcal{L}^T= J_y(x) H_\mathcal{L}(y)J_y^T(x)$.
 \end{lemma}

  \begin{proof}
    First, by chain rule, we have $\nabla_x\mathcal{L} = J_y(x)\nabla_y\mathcal{L}$.
    By cancellation of $J_y(x)$ on both sides, it suffices to show ${\nabla_y\mathcal{L}}\nabla_y\mathcal{L}^T=H_\mathcal{L}(y)$.
    Now, let us calculate the gradient:
    \begin{align}
     \begin{split}
    \pd{\mathcal{L}}{y_i}&=-\frac{t_i}{y_i}  \\
    &=-\frac{\delta_{il}}{y_i}
      \end{split}
    \end{align}
      Then, compute the out product:$
                ({\nabla_y\mathcal{L}}\nabla_y\mathcal{L}^T)_{ij}
                =\frac{\delta_{il}\delta_{jl}}{y_i y_j}
              $. This means all the off-diagonal items are zeros, and only one term left in the diagonal, which is $({\nabla_y\mathcal{L}}^T\nabla_y\mathcal{L})_{ll}=\frac{1}{y_l^2}$. Let us compute Hessian matrix:
                 \begin{align}
                  \begin{split}
       \he{\mathcal{L}}{y_j}{y_i} &=-\pd{}{y_j}\frac{\delta_{il}}{y_i}\\
              &=\frac{\delta_{il}\delta_{ij}}{y_i^2}
                \end{split}
                \end{align}
            Hence, $H_\mathcal{L}(y)$ also have only one diagonal term, which is $(H_\mathcal{L}(y))_{ll}=\frac{1}{y_l^2} $. Therefore, we have the equality as desired.
    \end{proof}
 \bibliography{example_paper}
 \bibliographystyle{IEEEtran}

\end{document}